\documentclass{article}
\usepackage[final]{cpal_2026}

\usepackage{url}
\usepackage{amsmath, amsfonts, amsthm}
\usepackage{csquotes, graphicx, xcolor, float}
\usepackage{bbm}
\usepackage{xcolor}
\usepackage{mathtools}
\usepackage{booktabs}

\theoremstyle{definition}

\theoremstyle{plain}

\newtheorem{proposition}{Proposition}
\theoremstyle{theorem}

\theoremstyle{remark}
\newtheorem{remark}{Remark}
\theoremstyle{corollary}

\usepackage{hyperref}
\hypersetup{colorlinks,
linkcolor={blue!80!black},
citecolor={blue!80!black},
urlcolor={blue!80!black}}

\newtheorem{theorem}{Theorem}
\newtheorem*{theorem*}{Theorem}

\newtheorem{lemma}{Lemma}
\newtheorem*{lemma*}{Lemma}

\usepackage{tikz}
\usetikzlibrary{arrows.meta}

\usepackage[capitalize,nameinlink]{cleveref}
\crefname{section}{Sec.}{Sec.}
\crefname{thm}{Thm.}{Theorem}
\crefname{appendix}{App.}{Appendices}
\crefname{algorithm}{Alg.}{Algorithms}
\crefname{equation}{Eq.}{Eqs.}
\crefname{figure}{Fig.}{Figs.}
\crefname{prop}{Prop.}{Props.}
\creflabelformat{equation}{#2\textup{#1}#3} 

\title{A Stein Identity for $q$-Gaussians\\ with Bounded Support}

\author{%
 Sophia Sklaviadis\textsuperscript{*,$\dagger$,1}, ~ Thomas M\"{o}llenhoff\textsuperscript{ *,2}, ~André F. T. Martins\textsuperscript{$\dagger$,3}, \\ ~Mário A. T. Figueiredo\textsuperscript{$\dagger$,4}, ~Mohammad Emtiyaz Khan\textsuperscript{*,5} \\ \ \\
  \textsuperscript{$*$}RIKEN Center for AI Project, Tokyo, Japan \\ \textsuperscript{$\dagger$}Instituto de Telecomunicações, Instituto Superior Técnico, Universidade de Lisboa, Portugal\\ \ \\
  \texttt{\textsuperscript{1}ssklaviadis@gmail.com, \textsuperscript{2}thomas.moellenhoff@riken.jp, \textsuperscript{3}andre.t.martins@tecnico.ulisboa.pt,  \textsuperscript{4}mario.figueiredo@tecnico.ulisboa.pt, \textsuperscript{5}emtiyaz.khan@riken.jp}
}

\definecolor{tablecolor}{rgb}{0.8,0.8,0.8}
\newcommand{\highlight}[1]{{\color{red} #1}}

\newcommand\cut[1]{}

\newcommand{\tSigma}{\widetilde{\vSigma}}

\newcommand{\squishlist}{
   \begin{list}{$\bullet$}
    { \setlength{\itemsep}{0pt}      \setlength{\parsep}{3pt}
      \setlength{\topsep}{3pt}       \setlength{\partopsep}{0pt}
      \setlength{\leftmargin}{1.5em} \setlength{\labelwidth}{1em}
      \setlength{\labelsep}{0.5em} } }

\newcommand{\squishlisttwo}{
   \begin{list}{$\bullet$}
    { \setlength{\itemsep}{0pt}    \setlength{\parsep}{0pt}
      \setlength{\topsep}{0pt}     \setlength{\partopsep}{0pt}
      \setlength{\leftmargin}{2em} \setlength{\labelwidth}{1.5em}
      \setlength{\labelsep}{0.5em} } }

\newcommand{\squishend}{
    \end{list}  }

{}
{}
{}

\newcommand{\half}{\mbox{$\frac{1}{2}$}}
\newcommand{\real}{\mbox{$\mathbb{R}$}}

\newcommand{\sqr}[1]{\left[#1\right]}

\newcommand{\myexpect}{\mathbb{E}}

\newcommand{\gauss}{\mbox{${\cal N}$}}

\newcommand{\myvec}[1]{\mbox{$\mathbf{#1}$}}
\newcommand{\myvecsym}[1]{\mbox{$\boldsymbol{#1}$}}

\newcommand{\vzero}{\mbox{$\myvecsym{0}$}}

\newcommand{\vepsilon}{\mbox{$\myvecsym{\epsilon}$}}

\newcommand{\vmu}{\mbox{$\myvecsym{\mu}$}}

\newcommand{\vSigma}{\mbox{$\myvecsym{\Sigma}$}}

\newcommand{\vg}{\mbox{$\myvec{g}$}}

\newcommand{\vu}{\mbox{$\myvec{u}$}}

\newcommand{\vw}{\mbox{$\myvec{w}$}}

\newcommand{\vx}{\mbox{$\myvec{x}$}}

\newcommand{\vz}{\mbox{$\myvec{z}$}}
\newcommand{\vA}{\mbox{$\myvec{A}$}}

\newcommand{\vE}{\mbox{$\myvec{E}$}}

\newcommand{\vH}{\mbox{$\myvec{H}$}}
\newcommand{\vI}{\mbox{$\myvec{I}$}}

\newcommand{\Var}{\mbox{$\operatorname{Var}$}}

\newcommand{\Cov}{\mbox{$\operatorname{Cov}$}}

\begin{document}

\maketitle

\begin{abstract}
Stein's identity is a fundamental tool in machine learning with applications in generative models, stochastic optimization, and other problems involving gradients of expectations under Gaussian distributions. Less attention has been paid to problems with non-Gaussian expectations.
Here, we consider the class of bounded-support $q$-Gaussians and derive 
a new Stein identity leading to gradient estimators which have nearly identical forms to the Gaussian ones, and which are similarly easy to implement.
We do this by extending the previous results of Landsman, Vanduffel, and Yao (2013) to prove new Bonnet- and Price-type theorems for \smash{$q$-Gaussians}. We also simplify their forms by using \emph{escort} distributions.
Our experiments show that bounded-support distributions can reduce the variance of gradient estimators, which can potentially be useful for Bayesian deep learning and sharpness-aware minimization.
Overall, our work simplifies the application of Stein's identity for an important class of non-Gaussian distributions. 
\end{abstract}

\section{Introduction}\label{sec:intro}

Stein's identity has been popular in machine learning for estimating gradients of $\myexpect_p[f(\vx)]$, where the expectation of a real-valued differentiable function $f:\mathbb{R}^D\rightarrow \mathbb{R}$ is taken with respect to a Gaussian density $p(\vx) = \gauss(\vx; \vmu,\vSigma)$ with mean $\vmu$ and covariance $\vSigma$. Applications arise in a wide-variety of machine-learning problems including stochastic optimization \citep{mohamed2020monte}, deep generative models \citep{rezende2014stochastic}, and variational inference \citep{blundell2015weight, khan2018fast}.
Stein's identity states that the following equality holds:
\begin{equation}\label{eq:gauss-stein}
\myexpect_p \left[ (\vx-\vmu)f(\vx) \right] = \Cov_p(\vx) \, \myexpect_p \left[ \nabla_{\text{\vx}} f(\vx) \right].
\end{equation}
First proved by Charles Stein \cite{Stein1973,stein1981estimation}, the identity results from integration by parts and holds under mild conditions on $f$ (see \citep[App. A.2]{lin2019stein}).

The Stein identity is used to express gradients with respect to $\vmu$ and $\vSigma$ in terms of the gradient and Hessian of $f(\vx)$, respectively (derivations are in \Cref{app:price-gauss}): 
\begin{equation}
    \label{eq:bonnet-price}
    \nabla_{\text{\vmu}} \, \myexpect_p \left[ f(\vx) \right] = \myexpect_p \left[ \nabla_{\text{\vx}} f(\vx) \right],
    \qquad\qquad
    \nabla_{\text{\vSigma}} \, \myexpect_p \left[ f(\vx) \right] = \half \myexpect_p \left[ \nabla^2_{\text{\vx}} f(\vx) \right].
\end{equation}
These expressions, known as Bonnet's \cite{bonnet1964transformations} and Price's \cite{price2003useful} theorems have received considerable attention in deep learning due to their convenient forms. Stochastic gradients with respect to $\vmu$ and $\vSigma$ are obtained by first sampling $\vx\sim p(\vx)$ and then computing $\nabla_{\text{\vx}} f(\vx)$ and $\nabla_{\text{\vx}}^2 f(\vx)$ at the sample. This line of thought has led to several variants of the ``reparameterization trick'' and other path-wise gradient estimators, which often have lower variance than score-function estimators \citep{williams1992simple}. Overall, Stein's identity has had a profound impact on stochastic-gradient estimation techniques \cite{liu2016stein}. 

Extensions of Stein's identity to non-Gaussian distributions have received much less attention than the Gaussian case. Generalizations do exist, for instance to elliptical families, which contain many interesting location-scale families, including the Gaussian distribution. In particular, \citet{landsman2008stein} and  \citet{landsman2013note,landsman2015some} have studied such generalizations focusing on the Pearson VII class, which contains various heavy-tailed distributions. The usefulness of such generalizations for gradient estimation has, to the best of our knowledge, not been explored. Can gradient estimators derived for non-Gaussian families also take simple forms and are they as easy to implement? The goal of this paper is answer these questions. 

In this paper, we derive a new Stein identity for the Pearson II class of elliptical families called bounded-support $q$-Gaussian distributions (\Cref{fig:loc-scale}). Such bounded-support distributions are interesting for gradient estimation because, unlike Gaussians, samples drawn from them always land inside a bounded interval, which also naturally implies a bound on the gradient variance.
We derive Bonnet- and Price-type theorems for this class of $q$-Gaussian distributions and show that the gradient estimators have nearly-identical form to those of Gaussian stochastic gradients, and are similarly easy to implement. This means we can compute unbiased gradient estimates by evaluating gradients and Hessians of $f(\vx)$ at samples $\vx \sim p(\vx)$. These gradient expressions are obtained by using a distribution associated with the base distribution known in information geometry and statistical physics \cite{naudts2009q, amari2011geometry, matsuzoe2011geometry} as the \emph{escort} distribution (\Cref{fig:loc-scale}). We present numerical experiments confirming the bounded variance property, and compare them to Bayesian deep-learning methods and to sharpness-aware minimization \cite{foret2020sharpness}. Overall, our work simplifies the application of Stein's identity to gradient estimation involving an important class of bounded-support $q$-Gaussian distributions.

\begin{figure}[!t]
    \centering
    \includegraphics[width=\linewidth]{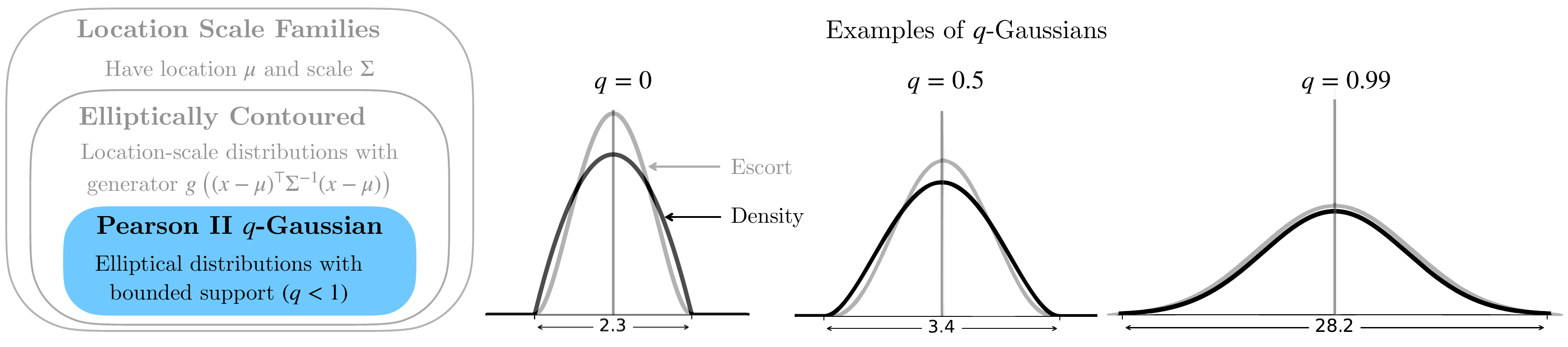}
    \caption{Bounded support $q$-Gaussians are a subclass of location-scale families that are elliptically contoured, that is, they are obtained by composing a quadratic \smash{$s(\vx) = (\vx-\vmu)^\top \vSigma^{-1}(\vx-\vmu)$} with a generator function $g$, thus $p(\vx) = g\bigl(s(\vx)\bigr)$. On the right, we show three examples of $q$-Gaussians for $q=0, 0.5,$ and $0.99$. We see for larger $q$ the base densities (black curves) are less peaked and have larger support. We also show the first associated \emph{escort} $2-q$-Gaussian densities (gray curves), which are slightly more peaked than their base densities. As $q\to 1$, $q$-Gaussians converge to Gaussians.}
    \label{fig:loc-scale}
\end{figure}

\section{Bounded-Support $q$-Gaussian Distributions}\label{sec:background}

We present a new Stein identity applied to gradient estimation involving bounded-support $q$-Gaussian distributions. Such distributions are a special type of location-scale distributions, obtained by applying a specific \emph{generator} function $g$ to a Gaussian-like quadratic form. We denote the location by $\vmu$ and the elliptical scale matrix by $\vSigma$. A $D$-variate elliptical density has the form
\begin{equation} \label{eq:elliptical_dist}
    p(\vx) = |\vSigma|^{-1/2} g\bigl( s(\vx) \bigr), \quad \text{ where } s(\vx) = (\vx - \vmu)^\top \vSigma^{-1}(\vx-\vmu),
\end{equation}
and $g:\mathbb{R}\rightarrow \mathbb{R}_+$ is a non-negative, density-generator function, appropriately scaled so that $p(\vx)$ is normalized, that is, it integrates to 1. For example, Gaussians are obtained by setting $g(s) \propto \exp(-s/2)$. The normalization condition can be written as 
\[
\int_0^\infty s^{D/2-1} g(s) ds = \frac{\Gamma(D/2)}{\pi^{D/2}},
\]
where $\Gamma$ is the Gamma function. If the second moments exists, the covariance matrix $ \Cov(\vx)$ is proportional to $\vSigma$. In the Gaussian case, $ \Cov(\vx) = \vSigma$, but otherwise $\vSigma$ should be thought of as a scale or dispersion parameter, which is proportional to the covariance.


Bounded-support $q$-Gaussians belong to a class of distribution known as  Pearson Type II \cite{Khalafi}, defined on the radius-$R$ ellipsoid $\{ \vx \in \real^D : (\vx - \vmu)^\top \vSigma^{-1}(\vx-\vmu)<R^2 \}$, where $R>0$. Given $R$, for all $s$ in the interval $0 \leq s < R^2$, the generator function takes the following form,
\begin{equation}\label{eq:genPearsonII}
g(s) = \frac{Z}{R^{ D+2m}} \left(R^2-s \right)^{ m}, \quad \text{ where } Z = \frac{1}{ \pi^{D/2} } \frac{\Gamma(\frac{D}{2}+m+1)}{\Gamma(m+1)}, 
\end{equation}
and $m$ is the shape parameter of the generator. We define $m = 1/(1-q)$ in terms of another scalar $q<1$, which gives rise to $q$-Gaussian designation that is well-known in information geometry and statistical physics \citep{naudts2009q, matsuzoe2011geometry, amari2011geometry}. The max support radius $R$ depends only on $q$ and the dimensionality $D$, as shown by the following lemma that gives an explicit form of $p(\vx)$ (the proof is in \Cref{app:elliptical}).

\begin{lemma}\label{lem:q-gaussian-pearsonII}
The density of the generalized Pearson Type~II subfamily with $q<1$ can be written as a $q$-Gaussian: 
\begin{equation}
    p(\vx) = \gauss_q(\vx|\vmu,\vSigma) = \exp_{q} \left[ \half |\vSigma|^{-\frac{1}{2m}}\left(  R^2-(\vx - \vmu)^\top \vSigma^{-1}(\vx-\vmu)\right)  - m \right],
\end{equation}
where $\exp_{q}(t) := \left[1+t/m \right]_+^m$ denotes the $q$-deformed exponential function \citep{NAUDTS2002323}, with $m = 1/(1-q)$ and $[ u ]_+ = \max( 0, u)$. The support radius $R$ is a function of $q$ and $D$ and is given by  
\begin{equation}\label{eq:R2q}
R^2 =\left[ \left(2m\right)^{m}  Z \right]^{2/(2m+D)} \; . 
\end{equation}
\end{lemma}
These distributions have recently been studied by \citet{martins2022sparse} as sparse continuous distributions induced by Tsallis entropies \cite{tsallis1988possible}. By using the definition of $\exp_q$, the same density can be expressed compactly as
\begin{equation}\label{eq:q-PearsonII-form1}
    p(\vx) \propto |\vSigma|^{-1/2} \left(R^2-s(\vx)\right)^{ m}_+,
\end{equation}
which will be useful later to highlight differences relative to the required \textit{escort} distribution.

\citet{landsman2008stein,landsman2013notel,landsman2015some} prove an extension of Stein’s identity to the class of absolutely continuous elliptical distributions, with a special focus on heavy-tailed distributions belonging to Pearson~Type VII family. Their results are based on canonical work on the properties of elliptical families \cite{cambanis1981theory, fang1990symmetric, arellano2001some, arellano2006bayesian}. Specifically, \citet[Prop.~2]{landsman2015some} give a general proof of elliptical Stein, but do not treat the bounded-support case on which we focus in this paper; see also \cite[Lemma~2]{landsman2013note} and \cite{landsman2008stein}. Notably, stochastic gradient estimators based on generalized elliptical Stein identities have not been studied and we are not aware of any works showing that such estimators can be brought into a form similar to those of Gaussian stochastic gradients. Our goal in this paper is to fill this gap.

\section{A Stein-type identity for bounded-support $q$-Gaussians}\label{sec:stein-q}

In this section we derive a Stein-type identity tailored to bounded-support $q$-Gaussians ($q<1$). We follow the approach of \citet{landsman2008stein, landsman2013note, landsman2015some}, using the \textit{associated} density $p^*$ to derive a Stein-type identity. We show in \cref{app:elliptical} that the first associated law coincides with the  $(2-q)$-\textit{escort} density $p^\star(\vx) \propto p(\vx)^{2-q}$ \cite{Naudts2004}. As far as we know this connection between the \textit{associated laws} defined in the classical statistical literature on elliptical families \citep{fang1990symmetric, johnson1987multivariate}, and the \textit{escort} distributions studied in statistical physics and information geometry \citep{matsuzoe2011geometry, Naudts2004} has not been noticed before. 
The use of escort distributions is instrumental in extending Stein's identity in an elegant way, mirroring the Gaussian Stein expression, and highlights the probabilistic structure that would otherwise be buried in repeated integration by parts. 

The associated density $p^\star$ is defined through the generator function that is obtained by integrating the original generator $g$ over the interval $(s, R^2)$,
\begin{equation}
    G(s) =\int_{s}^{R^2} g(t) dt. 
    \label{eq:gen-q}
\end{equation}
As we review in \Cref{app:elliptical}, using this generator yields another Pearson II distribution with exponent increased by 1,
\begin{equation}\label{eq:q-PearsonII}
p^\star(\vx) \propto |\vSigma|^{-1/2} \left(R^2-s(\vx) \right)^{m+1}_+. 
\end{equation}
Comparing this expression to \cref{eq:q-PearsonII-form1} shows that the only difference is that $m$ in the exponent is replaced by $m+1$. Note that $p$ and $p^\star$ share the same location-scale parameters $(\vmu,\vSigma)$ and bounded max support radius $R$, but $p^\star$ has a sharper peak at $\vmu$ because of the larger exponent. This difference is visible in \Cref{fig:loc-scale}.


In \cref{app:stein-proof} we prove the following theorem which establishes a new Stein-type identity for bounded-support $q$-Gaussian distributions.
\begin{theorem}[Bounded-support $q$-Gaussian Stein identity]\label{thm:stein-bounded}
For any almost everywhere differentiable \smash{$f:\real^D \to \real$}, with $\myexpect_{p^\star} \left\| \nabla f(\vx) \right\| < \infty$, when the second moment $\Cov_{p}(\vx)$ exists:
\begin{equation}\label{eq:stein-pearsonII-expect}
\myexpect_{p} \left[ (\vx-\vmu) f(\vx) \right] 
= \Cov_{p}(\vx) \; \myexpect_{\highlight{ p^\star }} \! \left[ \nabla_{\text{\vx}} f(\vx) \right].
\end{equation}

\end{theorem}

The form of \cref{eq:stein-pearsonII-expect} is nearly identical to that of \cref{eq:gauss-stein}, with the difference that the expectation on the right-hand side uses the escort $p^\star$ instead of the original $p$ (as highlighted in red).  The proof in \cref{app:stein-proof} proceeds by defining \smash{$\vz = \vSigma^{-1/2}(\vx-\vmu)$}, applying iterative one-dimensional integration by parts, and exploiting the fact that the Pearson~II density and its associated law vanish at the boundary, as well as the  facts described in \cref{lem:radial-moments-associated} below.
\begin{lemma} \label{lem:radial-moments-associated} 
The following facts hold for any bounded-support $q$-Gaussian with $q<1$,
\begin{enumerate}
    \item Defining $r(\vx) = \sqrt{s(\vx)}$, the following holds under $p$ and $p^\star$ respectively,
\[
    \frac{r(\vx)^2}{R^2} \sim \mathrm{Beta} \left( \tfrac D2, m + 1\right),
    \qquad \text{ and  } \qquad
    \frac{r(\vx)^2}{R^2} \sim \mathrm{Beta} \left( \tfrac D2, m + 2\right).
\]
    \item The expectations of $s(\vx) = r(\vx)^2$ have a closed-form expression in terms of $R$, $D$, and $m$,
\[
    \myexpect_{p} \left[ s(\vx) \right] = \frac{D\, R^2}{ D+2(m+1) } ,
    \qquad \text{ and  } \qquad
    \myexpect_{p^\star} \left[ s(\vx) \right] = \frac{D\, R^2}{ D+2 \left(m+2 \right) } . 
\]
    \item The second moment can be written in terms of $\vSigma$,
\[
    \Cov_{p}(\vx) = \frac{1}{D} \myexpect_{p} [s(\vx)] \vSigma,
    \qquad \text{ and  } \qquad
    \Cov_{p^\star}(\vx) = \frac{1}{D} \myexpect_{p^\star} [s(\vx)] \vSigma.
\] 
    \item Finally, we have the following reweighted representation of $p^\star$,
\[
p^\star(\vx) = \frac{(R^2 - s(\vx))\, p(\vx)}{\myexpect_{p} \left[ R^2-s(\vx) \right]}  .
\]
\end{enumerate}
\end{lemma}
All of these results follow from the application of the definitions of the Beta expectation and the covariance. \cref{app:elliptical} contains additional details. 

There are two other useful variants of the identity,  
\begin{align}
\myexpect_{p} \left[ (\vx-\vmu) f(\vx) \right] 
&= \highlight{ \frac{1}{D} \myexpect_{p} \left[ s(\vx) \right] } \,\, \vSigma \; \myexpect_{\highlight{ p^\star } } \left[ \nabla_{\text{\vx}} f(\vx) \right] \label{eq:stein-pearsonII-expect_1}\\
&= \highlight{ \frac{1}{D} \myexpect_{p} \left[ s(\vx) \right] } \,\, \vSigma\; 
  \frac{\myexpect_{p} \left[ \highlight{ (R^2-s(\vx)) } \nabla_{\text{\vx}} f(\vx) \right]}{ \highlight{ \myexpect_{p} \left[ R^2-s(\vx) \right] } } \label{eq:final-p-only}.
\end{align}
In the first line, we expand the second moment in terms of $\vSigma$, and in the second line, we rewrite the $p^\star$-expectation in terms of the base density $p(\vx)$ by using the last fact of \cref{lem:radial-moments-associated}. The form of \cref{eq:final-p-only} is useful specifically for implementation because it expresses the identity in terms of $p(\vx)$ only. Sampling from $p(\vx)$ is comparably efficient to sampling from a Gaussian, as we show next.

Efficient sampling from $p(\vx)$ is possible by using \cref{lem:radial-moments-associated}. We define \smash{$\vz(\vx) = \vSigma^{-1/2}(\vx-\vmu)$}, which can be rewritten in terms of an independent random variable $\vu$ that is uniformly distributed on a sphere in $\mathbb{R}^D$ as $\vz(\vx) = r(\vx) \vu$, where $\vu \sim \mathrm{Unif}(S^{D-1})$ and $S^{D-1} = \{ \vu \in \mathbb{R}^D : \| \vu\|_2 = 1 \}$. Further, $r(\vx)^2/R^2$ is an independent Beta random variable. Thus sampling proceeds in the following four steps:
\begin{equation} \label{eq:sampling}
    \vu \sim \mathrm{Unif}(S^{D-1}), 
    \qquad
    \frac{r^2}{R^2} \sim \mathrm{Beta} \left( \tfrac D2, m + 1\right),
    \qquad
    \vz \gets r \vu,
    \qquad
    \vx \gets \vmu + \vSigma^{1/2}\vz. 
\end{equation}
The radial parametrization reviewed in \Cref{app:elliptical}.

\section{Bonnet- and Price-type Theorems}\label{sec:bonnet-price-q}

We state Bonnet- and Price-type theorems for bounded-support $q$-Gaussians, for which proofs can be found in \cref{app:q-bonnet-proof}. These theorems are brought into forms that are either identical or structurally analogous to the corresponding Gaussian theorems. The stochastic gradient expressions we derive are easy to estimate using Monte-Carlo sampling since the distributions over which the expectations are defined are easy to sample from by exploiting \cref{eq:sampling}.  

We begin with the Bonnet-type theorem, paralleling the first equality in \cref{eq:bonnet-price}. 
\begin{theorem}[$q$-Bonnet]\label{thm:q-bonnet}
For bounded-support $p(\vx) = \mathcal{N}_q(\vx|\vmu,\vSigma)$, assume $f:\real^D \to \real$ to be $C^1$ on an open set containing $\{\vx:\, s(\vx) \le R^2\}$, and $\myexpect_{p} \left\| \nabla f(\vx)\right\| < \infty$. Then,
\begin{equation}\label{eq:q-bonnet}
\nabla_{\text{\vmu}} \myexpect_{p} \left[ f(\vx) \right] = \myexpect_{p} \left[ \nabla f(\vx) \right].
\end{equation}
\end{theorem}
This identity has exactly the same form as the first equality in \cref{eq:bonnet-price}. 
The proof is in \Cref{app:q-bonnet-proof-1} and follows from the form of $p(\vx)$ given in \cref{eq:q-PearsonII-form1} and the fact that $\nabla_{\text{\vmu}} s(\vx) = -2 \vSigma^{-1}(\vx - \vmu)$ which allows us to write the gradient as
\[
\nabla_{\text{\vmu}} \log p(\vx) 
= \frac{2m}{R^2-s(\vx)} \vSigma^{-1}(\vx - \vmu).
\]
Differentiating under the integral, the left-hand side in \cref{eq:q-bonnet} can be written in terms of $\myexpect_{p} \left[ f(\vx)\vSigma^{-1}(\vx-\vmu)/ (R^2-s(\vx))  \right]$. Finally, the Stein-type identity in \cref{eq:final-p-only} and some algebra lead to the desired result. 

Next, we state the Price-type theorem which parallels the second equality in \cref{eq:bonnet-price}:
\begin{theorem}[$q$-Price]\label{thm:q-price}
For bounded-support $p(\vx) = \mathcal{N}_q(\vx|\vmu,\vSigma)$, assume $f:\real^D \to \real$ is $C^2$ on an open set containing $\{ s(\vx) \le R^2\}$ and that $\myexpect_p \sqr{ \| \nabla f(\vx) \| } + \myexpect_p \sqr{ \| \nabla^2 f(\vx) \|_{F}}~<~\infty$. Then,
\begin{equation}\label{eq:q-price-matrix}
    \nabla_{ \text{\vSigma} } \myexpect[f(\vx)]
    = \highlight{ \frac{1}{D} \; \myexpect_p[s(\vx)]} \;\frac{1}{2}\; \myexpect_{\highlight{ p^\star} }[\nabla_{\text{\vx}}^2 f(\vx)].
\end{equation}
In the Gaussian limit $q \uparrow 1$ (thus $m \to \infty$ and $R \to \infty$), $p^\star \rightarrow p$ and $\myexpect_p[r(\vx)^2]/D\to 1$, which implies that \cref{eq:q-price-matrix} reduces to the classical Price theorem
$\frac{\partial}{\partial \text{\vSigma}_{ij}}\myexpect[f(\vx)]=\tfrac12 \myexpect[\partial_{x_i}\partial_{x_j} f(\vx)]$.
\end{theorem}
A proof is provided in \Cref{app:q-bonnet-proof-2}. This Price-type theorem is very similar to the Gaussian Price theorem in the second equation of \cref{eq:bonnet-price}, and differences are highlighted in red. The right-hand expectation is taken with respect to the $(2-q)$-escort $p^\star$, and there is an additional factor $\myexpect_p[s(\vx)]/D$. The expression is easy to evaluate because sampling from $p^\star$ is efficient as shown in \cref{eq:sampling}.

\section{Applications}

\subsection{Bounded-variance Monte Carlo estimators}\label{subsec:bounded-variance}

An advantage of bounded-support distributions is that they lead naturally to gradient estimators with bounded variance. We give a formal statement of bounds on the variance of the estimators based on \cref{eq:final-p-only} and \cref{eq:q-price-matrix} respectively. The Stein-type identity in \cref{thm:stein-bounded} and the $q$-Price \cref{thm:q-price} express the gradients of interest in terms of escort expectations as $\myexpect_{p^\star}[\nabla f(\vx)]$ and $\myexpect_{p^\star} [\nabla^2 f(\vx)]$. Since we are interested in approximating these expectations by Monte Carlo, the variance of the resulting estimators directly affects the stochastic gradient noise. Using the reweighted $p$-only expressions from \cref{eq:reweight} in \cref{app:elliptical}, under mild boundedness assumptions on $\nabla f$ and  $\nabla^2 f$, we can bound the variance of the stochastic gradient MC estimators.

\begin{proposition}[Bounded variance MC estimators]\label{prop:bv} Let $\vx_1, ... \vx_S$ be iid samples from a $q$-Gaussian $p(\vx)$. For any almost everywhere differentiable $t:\real^D \to \real$ with $\myexpect_{p^\star} \sqr{ \left\| \nabla t(\vx) \right\|} < \infty$, and $f:\real^D \to \real$ that is $C^2$ on an open set containing $\{s(\vx) \le R^2\}$ such that $\myexpect_p \sqr{ \| \nabla f(\vx) \| } + \myexpect_p \sqr{ \| \nabla^2 f(\vx) \|_{F} }~<~\infty$, define the following Monte-Carlo estimators:
\[
\myexpect_{p^\star} [\nabla t(\vx)] \approx \widehat \vg  = \frac{1}{S} \sum_{k=1}^S \frac{(R^2-s(\vx_k)) \nabla t(\vx_k)}{\myexpect_{p} \left[ R^2-s(\vx) \right]},
\qquad
\myexpect_{p^\star} [\nabla^2 f(\vx)] \approx \widehat \vH = \frac{1}{S} \sum_{k=1}^S \frac{(R^2-s(\vx_k)) \nabla^2 f(\vx_k)}{\myexpect_{p} \left[ R^2-s(\vx) \right]}. 
\]
Note that $M = \myexpect_p \left[ R^2 - s(\vx) \right]$ is available in closed form from \cref{lem:radial-moments-associated}. Assume there exist finite constants $C_1, C_2$ such that 
\[
\underset{ \{s(\text{\vx}) < R^2 \} }{\mathrm{sup}} \| \nabla t(\vx) \| \le C_1,
\qquad \text{ and } \qquad
\underset{ \{s(\text{\vx}) < R^2 \} }{\mathrm{sup}} \| \nabla^2 f(\vx) \|_{\mathrm{op}} \le C_2,
\]
where $\| \cdot \| $ is the spectral (operator) norm on matrices. Then, for each entry $\widehat{g_j}$ and $\widehat{H_{ij}}$ we have,
\[
\Var \left(\widehat g_j \right)
  \le \frac{1}{S} \left( \frac{R^2 C_1}{M} \right)^2,
\qquad \text{ and } \qquad
\Var \left((\widehat H)_{ij}\right)
  \le \frac{1}{S} \left( \frac{R^2 C_2}{M} \right)^2.
\]
Consequently,
\[
\myexpect \left[ \big\| \widehat \vH - \myexpect \left[ \widehat \vH \right] \big\|_F^2 \right]
  \le D^2 \frac{1}{S} \left( \frac{R^2 C_2}{M} \right)^2,
\qquad \text{ and } \qquad
\myexpect \left[ \big\| \widehat \vH-\myexpect \left[ \widehat \vH \right] \big\|_{\mathrm{op}} \right]
 \le \frac{C_3 R^2 C_2}{M}
      \sqrt{\frac{\log D}{S}}
\]
for $M = \myexpect_p \left[ R^2 - s(\vx) \right]$ and finite constant $C_3$.
\end{proposition}
\begin{proof}
The proof is in \cref{app:q-bonnet-proof-3}. It follows from bounded‑range variance bounds (Popoviciu’s inequality) and matrix Hoeffding.   
\end{proof}

\begin{figure}[t!]
\includegraphics[width=0.5\linewidth]{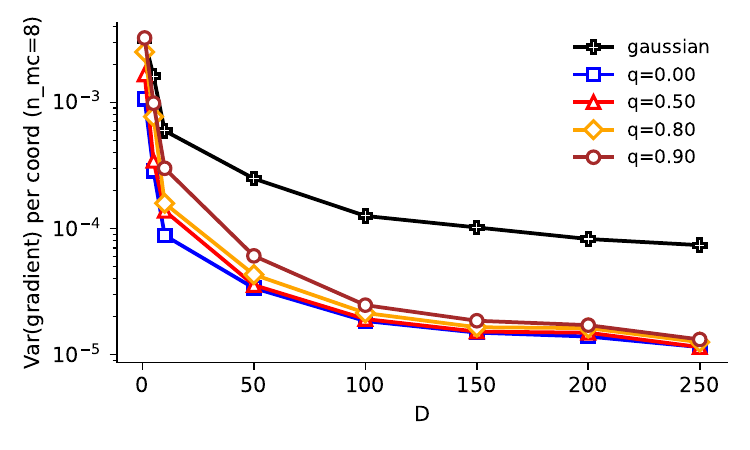}
\includegraphics[width=0.5\linewidth]{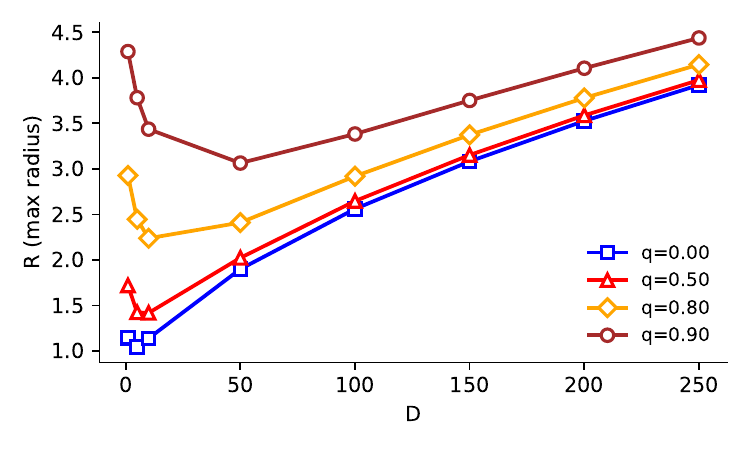}
\caption{Synthetic logistic regression. Left: For $D \in \{10, 50, 200\}$ and $q \in \{0.0, 0.5, 0.8, 1\}$, we draw 8 Monte Carlo samples and compute the empirical per-coordinate gradient variance $\frac{1}{D}\sum_{j=1}^D \Var(\widehat{\nabla} F(\vw^\star)_j)$, averaged over $50$ independent repetitions; standard errors are all <.003. Right: Maximum radius of the $q$-Gaussian support against $D$. }\label{fig:logreg_q_gauss}
\label{fig:logreg}
\end{figure}
\subsection{Numerical experiments}

\paragraph{Synthetic logistic regression experiment.} For dimensions $D \in \{10, 50, 200\}$, we draw a dataset $\{(\vx_i, y_i)\}_{i=1}^N$ with $\vx_i \sim \mathcal{N}(0, \vI_D)$, a ground-truth weight vector $\vw^\star \sim \mathcal{N}(\vzero, \vI_D)$, and labels $y_i \mid \vx_i \sim \mathrm{Bernoulli}(\sigma(\vx_i^\top \vw^\star))$, where $\sigma$ is the sigmoid function. The loss $f(\vw)$ is the binary cross-entropy loss. We consider Monte Carlo pathwise estimators of the gradient of the objective evaluated at $\vw = \vw^\star$,
\[
F(\vw)= \myexpect_{\text{\vepsilon}}\left[f(\vw + \vepsilon)\right].
\]
For the Gaussian baseline we take $\vepsilon \sim \mathcal{N}(0, \vI_D)$. For bounded-support $q$-smoothing, we draw $\vepsilon$ from an isotropic $D$-dimensional bounded-support $q$-Gaussian, so that $\|\vepsilon\|_2 \le R(D,q)$. For each $D\in\{10, 50, 200\}$ and $q \in \{0.0, 0.5, 0.8, 1\}$, we draw 8 Monte Carlo samples $\vepsilon$ and compute
\[
\widehat{\nabla} F(\vw^\star) = \frac{1}{S}\sum^{S}
\nabla_{\text{\vw}} f(\vw),
\]
for $\vw = \vw^\star + \vepsilon$. In \Cref{fig:logreg} (left), for each $q$, over $50$ independent repetitions we plot the empirical per-coordinate gradient variance $\frac{1}{D}\sum_{j=1}^D \Var(\widehat{\nabla} F(\vw^\star)_j)$ against $D$. 
The results in \Cref{fig:logreg} (left) confirm that smaller values of $q$ lead to lower-variance gradient estimators.
We also plot the maximum radius of the support against $D$ in \Cref{fig:logreg} (right) to visualize how the radius increases with the dimension and as $q$ gets closer to one; the dimension noticeably dominates the magnitude of $R$ and the effect of $q$ is significant only with very small $D$.

\begin{figure}
    \centering
    \includegraphics[width=\linewidth]{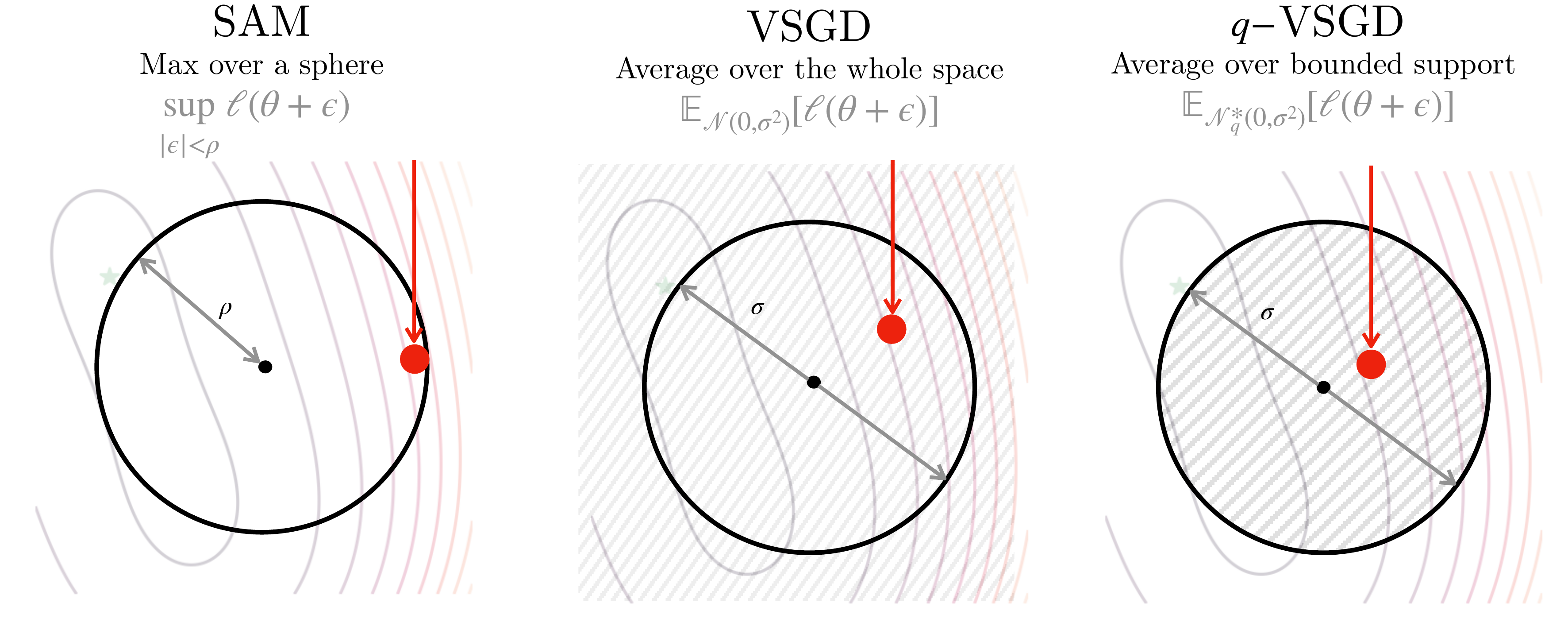}
    \caption{Sharpness-Aware Minimization (SAM)~\citep{foret2020sharpness} considers an adversarial perturbation over a compact ball. Variational stochastic gradient descent (VSGD) with Gaussian weight perturbation averages perturbations over whole space (unbounded support). Our proposed $q$-VSGD uses $q$-Gaussian weight perturbations, which have bounded support similarly to SAM but uses averages similarly to VSGD. The method combines the two complementary features of SAM and VSGD.} 
    \label{fig:placeholder}
\end{figure}

\paragraph{Variational SGD with $q$-Gaussian noise.}
Next, we explore the potential use of our gradient estimators in Bayesian deep learning to estimate a $q$-Gaussian posteriors instead of the usual Gaussian ones. We compare with several variational training methods and also with the Sharpness-Aware Minimization (SAM) algorithm. A visualization is in \Cref{fig:placeholder}, illustrating the main differences which are related to the choice of the point at which gradients are evaluated.

In SAM, we consider a point in a circle around the current parameter that has the worst loss value. Formally, let $f(\vw;\mathcal{B})$ be the minibatch cross-entropy, where $\mathcal{B}$ is the minibatch. Then, we set the perturbation to be $\delta = \rho \vg/\|\vg\|_2$ with $\vg = \nabla_{\text{\vw}} f(\vw;\mathcal{B})$. The parameters are then updated by evaluating the gradient at this point,
\[
\vw \leftarrow \vw - \eta \nabla_{\text{\vw}} f(\vw + \delta; \mathcal{B}).
\]
Overall, this requires two gradient evaluations, using in total two forward and backward pass.
In Variational SGD (VSGD), the goal is to estimate the mean of an isotropic Gaussian posterior. The mean is updated by taking the gradient of the expected loss, requiring at least one MC sample. This means that the update is the same as above but we need to use $\delta \sim \gauss(0, \vI_D)$. 

We propose a new variant of VSGD where we aim to estimate the mean of a $q$-Gaussian posterior, while fixing $\vSigma$ to $\vI_D$. This can be done by making a simple change to VSGD: instead of sampling from a standard normal, we sample from a $q$-Gaussian as shown below:  
\[
\delta \gets \rho \frac{\vepsilon}{R(q,D)}, \text{ where } \vepsilon \sim \gauss_q(\vzero, \vI_D), 
\]
and $R(q,D)$ is the distribution-dependent radius used for normalization. Note that $q$-VSGD with $q=1$ is equivalent to VSGD.
We also compare $q$-VSGD to an optimizer called Improved Variational Online Newton (IVON) which uses a more flexible diagonal-Gaussian posterior \citep{shen2402variational}. 

We compare the methods on a standard ResNet-20 model on the CIFAR-10 dataset, where we find that changing $q$ leads to small improvements over VSGD. The experimental setup follows the one used in the paper that proposed IVON~\cite{shen2402variational}. We train for 200 epochs with batch size 50 using SGD with momentum $0.9$ and weight decay $10^{-4}$, a 5-epoch linear warmup followed by cosine annealing, and sweep over values of $q\leq 1$ as reported in \Cref{tab:cifar10}. In \Cref{tab:cifar10} we report test accuracy, NLL, Brier score, ECE (20 bins), AUROC (as described in~\citep{osawa2019practical}), and wall-clock seconds per epoch (mean $\pm$ standard error over 10 seeds). For $q=0.6$, we observe some improvement in the accuracy. We can further increase the accuracy by increasing the number of MC samples, but the algorithm becomes slower. Overall, the results are not conclusive and indicate that tuning $q$ may not directly yield increased performance right away, despite the fact that the variance of gradient estimator is bounded. 

One potential reason behind the lack of better performance is that for very large dimensions the effect of varying $q$ reduces drastically. This happens because the support is heavily influenced by dimensionality, as shown in \cref{fig:logreg_q_gauss} where for high dimensions all $q$ values have very similar support. Our numerical results suggest using a more flexible $q$-Gaussian form, for example, estimating the scale parameter. We expect a $q$-IVON algorithm will perform better than $q$-VSGD, which does outperform VSGD. Perhaps the most effective modification is to use a low-dimensional $q$-Gaussian factorization. We hope to explore these alternatives in future work.

\begin{table}[t!]
\centering
\resizebox{\textwidth}{!}{
\begin{tabular}{lcccccc}
\toprule
Method & Acc. (\%) $\uparrow$ & NLL $\downarrow$ & ECE (\%) $\downarrow$ & Brier $\downarrow$ & AUROC $\uparrow$ & Time $\downarrow$ \\
\midrule
SGD & 92.0 $\pm$ 0.1 & 0.28 $\pm$ 0.00 & 3.84 $\pm$ 0.11 & 0.12 $\pm$ 0.00 & 92.1 $\pm$ 0.1 & 37s \\
IVON (1-MC) ~\cite{shen2402variational} & 92.5 $\pm$ 0.1 & 0.26 $\pm$ 0.00 & 3.43 $\pm$ 0.09 & 0.12 $\pm$ 0.00 & 92.5 $\pm$ 0.1 & 43s \\
SAM~\cite{foret2020sharpness} & 92.6 $\pm$ 0.1 & 0.22 $\pm$ 0.00 & 1.56 $\pm$ 0.05 & 0.11 $\pm$ 0.00 & 92.5 $\pm$ 0.1 & 70s \\
\midrule
\multicolumn{7}{c}{1 Monte-Carlo Sample}\\
\midrule
$q$-VSGD ($q=0.0$) & 92.1 $\pm$ 0.0 & 0.25 $\pm$ 0.00 & 3.01 $\pm$ 0.05 & 0.12 $\pm$ 0.00 & 92.3 $\pm$ 0.0 & 44s \\
$q$-VSGD ($q=0.2$) & 92.1 $\pm$ 0.0 & 0.26 $\pm$ 0.00 & 2.99 $\pm$ 0.04 & 0.12 $\pm$ 0.00 & 92.3 $\pm$ 0.1 & 43s \\
$q$-VSGD ($q=0.4$) & 92.0 $\pm$ 0.0 & 0.26 $\pm$ 0.00 & 3.10 $\pm$ 0.05 & 0.12 $\pm$ 0.00 & 92.3 $\pm$ 0.1 & 43s \\
$q$-VSGD ($q=0.6$) & 92.2 $\pm$ 0.0 & 0.25 $\pm$ 0.00 & 2.93 $\pm$ 0.04 & 0.12 $\pm$ 0.00 & 92.3 $\pm$ 0.1 & 43s \\
$q$-VSGD ($q=0.8$) & 92.1 $\pm$ 0.0 & 0.26 $\pm$ 0.00 & 3.02 $\pm$ 0.04 & 0.12 $\pm$ 0.00 & 92.2 $\pm$ 0.1 & 43s \\
VSGD & 92.1 $\pm$ 0.0 & 0.25 $\pm$ 0.00 & 2.79 $\pm$ 0.05 & 0.12 $\pm$ 0.00 & 92.4 $\pm$ 0.0 & 42s \\
\midrule
\multicolumn{7}{c}{5 Monte-Carlo Samples}\\
\midrule
$q$-VSGD ($q=0.0$) & 92.6 $\pm$ 0.1 & 0.25 $\pm$ 0.00 & 3.06 $\pm$ 0.08 & 0.11 $\pm$ 0.00 & 92.5 $\pm$ 0.1 & 173s \\
$q$-VSGD ($q=0.2$) & 92.5 $\pm$ 0.0 & 0.25 $\pm$ 0.00 & 3.13 $\pm$ 0.07 & 0.11 $\pm$ 0.00 & 92.6 $\pm$ 0.1 & 173s \\
$q$-VSGD ($q=0.4$) & 92.4 $\pm$ 0.0 & 0.25 $\pm$ 0.00 & 3.17 $\pm$ 0.06 & 0.11 $\pm$ 0.00 & 92.7 $\pm$ 0.1 & 174s \\
$q$-VSGD ($q=0.6$) & 92.6 $\pm$ 0.1 & 0.25 $\pm$ 0.00 & 3.04 $\pm$ 0.07 & 0.11 $\pm$ 0.00 & 92.5 $\pm$ 0.1 & 174s \\
$q$-VSGD ($q=0.8$) & 92.4 $\pm$ 0.0 & 0.25 $\pm$ 0.00 & 3.17 $\pm$ 0.05 & 0.11 $\pm$ 0.00 & 92.7 $\pm$ 0.1 & 182s \\
VSGD & 92.4 $\pm$ 0.1 & 0.25 $\pm$ 0.00 & 3.13 $\pm$ 0.09 & 0.12 $\pm$ 0.00 & 92.6 $\pm$ 0.1 & 173s \\
\bottomrule\\
\end{tabular}
}
\caption{ResNet-20 on CIFAR-10 test performance. Mean $\pm$ standard error over 10 seeds. Among the baselines (the first block at the top), SGD is the worst but fastest. SAM achieves the highest accuracy but is also slow due to two gradients needed in each iteration, while IVON performs reasonably well with a reasonable cost. In the second block, we show $q$-VSGD with 1 Monte-Carlo sample which achieves slightly better accuracy than its counterpart VSGD by increasing $q=0.6$, although it is comparable in other metrics. Going to 5-MC in the third block improves the performance but slows down the algorithm. The results are mixed and indicate more work is needed in improving performance via tweaking $q$.}
\label{tab:cifar10}
\end{table}

\section{Conclusion}
This paper proves a new Stein-type identity for bounded-support $q$-Gaussians, which belong to the subfamily of the Pearson II class of elliptical distributions. We show that the associated (cumulative-generator) law has a simple escort interpretation, allowing for the derivation of Bonnet- and Price-type identities with forms closely resembling those of the Gaussian stochastic gradient estimators. These results are applied to practical pathwise gradient estimators that are easy to implement by sampling from the $q$-Gaussian density. A key consequence of the support's boundedness is that the resulting Monte Carlo estimators have simple bounded variance guarantees. Our experiments illustrate these effects in a controlled synthetic setting, and in deep neural models trained on CIFAR-10, where bounded-support perturbations are competitive with SAM and offer a principled distributional analogue of bounded-radius perturbation methods. 

The results in this paper open new directions for research to exploit generalized Stein identities in broader stochastic gradient applications, such as variational inference and robust optimization, including extensions beyond bounded-support $q$-Gaussians (notably to the heavy-tailed $3>q>1$ regime), learning or adapting $R$, and considering anisotropic $\vSigma$. 

\section*{Acknowledgements}
SS, MEK, and TM were supported by the Bayes duality project, JST CREST Grant Number JPMJCR2112. 
AM was supported by the project DECOLLAGE (ERC-2022-CoG 101088763). 
AM and MF were supported by the Portuguese Recovery and Resilience Plan through project C64500888200000055
(Center for Responsible AI) and by Fundação para a Ciência e Tecnologia through contract
UIDB/50008/2020.

\bibliography{reference}


\appendix

\section{Proof of Bonnet's and Price's theorems with Gaussians}\label{app:price-gauss}

\paragraph{Gaussian Stein, Bonnet, and Price.}
Let $\vx \sim \gauss(\vmu,\vSigma)$ with density $p(\vx)$. The general Stein identity states that for any differentiable test function $t:\real^D\to\real$ with $\myexpect \left\| t(\vx)\right\| + \myexpect \left\| \nabla_{\text{\vx}} t(\vx)\right\| < \infty$,
\begin{equation}\label{eq:gauss-stein-1}
\myexpect \left[ (\vx-\vmu)t(\vx) \right] = \vSigma\myexpect \left[ \nabla_{\text{\vx}} t(\vx) \right].
\end{equation}
This follows from the score identity  $\myexpect \left[ t(\vx) \nabla_{\text{\vx}} \log p(\vx) \right] = -\myexpect \left[ \nabla_{\text{\vx}} t(\vx) \right]$ through integration by parts,
and the Gaussian-specific fact that $\nabla_{\text{\vx}}\log p(\vx) = -\vSigma^{-1}(\vx-\vmu)$. 

The score identity implies: 
\[
\begin{rcases*}
\nabla_{\text{\vmu}} \myexpect \left[ f(\vx) \right] = \myexpect \left[ f(\vx) \nabla_{\text{\vmu}} \log p(\vx; \vmu,\vSigma)\right] \\
\nabla_{\text{\vmu}} \log p(\vx; \vmu,\vSigma) = \vSigma^{-1}(\vx-\vmu)
\end{rcases*} \implies \\
\nabla_{\text{\vmu}} \myexpect \left[ f(\vx) \right] = \vSigma^{-1}\myexpect \left[ (\vx-\vmu)f(\vx) \right].
\]
Applying the Stein identity with $t(\vx) = f(\vx)$, we obtain Bonnet's theorem 
\[
\nabla_{\text{\vmu}} \myexpect \left[ f(\vx) \right] = \vSigma^{-1} \vSigma \myexpect \left[ \nabla_{\text{\vx}} f(\vx) \right] = \myexpect \left[ \nabla_{\text{\vx}} f(\vx) \right]. 
\]
Writing $\vx = \vmu + \vSigma^{1/2} \vepsilon$ with $\vepsilon \sim \gauss(\vzero, \vI)$, since $\partial \vx / \partial \vmu = \vI$, we see that a popular case of the location-scale transform or ``reparameterization trick'' is the pathwise implementation of Bonnet. 
Similarly, Price's theorem, which states that for differentiable $f$ we have $\nabla_{\text{\vSigma}} \myexpect \left[ f(\vx) \right] = \frac{1}{2} \myexpect \left[ \nabla^2_{\text{\vx}} f(\vx) \right]$, follows by applying the same Stein identity to the components of $\nabla_{\text{\vx}} f(\vx)$ and using the score identity for derivatives with respect to $\vSigma$.

\paragraph{Gaussian Stein and Price's theorem.} Let $\vx \sim \gauss(\vmu, \vSigma)$ on $\real^D$ and let $f: \real^D \to \real$ be twice continuously differentiable with $\myexpect \left\| \nabla_{\text{\vx}} f(\vx)\right\| < \infty$ and $\myexpect \left\| \nabla_{\text{\vx}}^2 f(\vx)\right\|_{\mathrm{F}} < \infty$. The scalar Gaussian Stein identity states that for any sufficiently regular $t: \real^D \to \real$ and each coordinate $i=1,..., D$,
\[
  \myexpect \left[ ( \vx_i - \vmu_i) t(\vx)\right] = \sum_{\ell=1}^D \Sigma_{i\ell} \myexpect \left[ \partial_{x_\ell} t(\vx)\right].
\tag{Stein}
\]
\begin{proof}[Proof of Price]
To differentiate with respect to $\vSigma$, we use the matrix score identity for the Gaussian density:
\begin{equation}
  \nabla_{\text{\vSigma}} \log p( \vx; \vmu, \vSigma) =  -\tfrac12 \vSigma^{-1} + \tfrac12 \vSigma^{-1}(\vx-\vmu)(\vx-\vmu)^\top \vSigma^{-1},
\end{equation}
so, entrywise,
\begin{align}\label{eq:price-score}
  \frac{\partial}{\partial \vSigma_{ij}} \myexpect \left[ f(\vx) \right] 
  &=  \myexpect \left[ f(\vx) \frac{\partial}{\partial \text{\vSigma}_{ij}} \log p(\vx; \vmu,\vSigma) \right] \\
  &= -\tfrac12(\vSigma^{-1})_{ji} \myexpect \left[ f(\vx) \right]  + \tfrac12 \left(\vSigma^{-1} \myexpect \left[ (\vx-\vmu)(\vx-\vmu)^\top f(\vx)\right] \vSigma^{-1}\right)_{ji}.
\end{align}
For each fixed $j=1,..., D$, apply Stein with the scalar test function $t(\vx) := (\vx_j - \vmu_j) f(\vx)$. For each $i$,
\[
  \myexpect \left[ (\vx_i - \vmu_i)(\vx_j - \vmu_j) f(\vx) \right] = \sum_{\ell=1}^D \vSigma_{i\ell} \myexpect \left[ \partial_{\text{\vx}_\ell}\left((\vx_j-\vmu_j) f(\vx) \right)\right].
\]
Since
\[
  \partial_{\text{\vx}_\ell} \left((\vx_j - \vmu_j) f(\vx) \right) = \delta_{\ell j} f(\vx) + (\vx_j - \vmu_j) \partial_{\text{\vx}_\ell} f(\vx),
\]
where $\delta_{\ell j}$ is the Kronecker delta, we get
\begin{align}    
  \myexpect \left[ (\vx_i - \vmu_i)(\vx_j - \vmu_j) f(\vx) \right] 
  &=  \sum_{\ell} \vSigma_{i\ell} \delta_{\ell j} \myexpect \left[ f(\vx) \right] + \sum_{\ell} \vSigma_{i\ell} \myexpect \left[ (\vx_j - \vmu_j) \partial_{\text{\vx}_\ell} f(\vx) \right]\\
  &= \vSigma_{ij} \myexpect \left[ f(\vx) \right] + \sum_{\ell} \vSigma_{i\ell} \myexpect \left[ (\vx_j - \vmu_j) \partial_{\text{\vx}_\ell} f(\vx)\right].
\end{align}
Next, apply Stein again with the scalar test function $t(\vx) := \partial_{\text{\vx}_\ell} f(\vx)$. For each $\ell=1,..., D$,
\[
  \myexpect \left[ (\vx_j - \vmu_j) \partial_{\text{\vx}_\ell} f(\vx) \right]
  = \sum_{k=1}^D \vSigma_{jk}  \myexpect \left[ \partial_{\text{\vx}_k} \partial_{\text{\vx}_\ell} f(\vx)\right].
\]
Substituting into the previous expression yields
\[
  \myexpect \left[ (\vx_i - \vmu_i)(\vx_j - \vmu_j) f(\vx)\right] = \vSigma_{ij} \myexpect \left[ f(\vx) \right]
  + \sum_{\ell,k} \vSigma_{i\ell} \vSigma_{jk} \myexpect \left[ \partial_{\text{\vx}_k} \partial_{\text{\vx}_\ell} f(\vx)\right].
\]
In matrix notation this is
\begin{equation}\label{eq:price-middle}
  \myexpect \left[ (\vx-\vmu)(\vx-\vmu)^\top f(\vx)\right] = \vSigma \myexpect \left[ f(\vx) \right] + \vSigma \myexpect \left[ \nabla_{\text{\vx}}^2 f(\vx)\right] \vSigma.
\end{equation}
Multiply \eqref{eq:price-middle} on the left and right by $\vSigma^{-1}$:
\[
  \vSigma^{-1} \myexpect \left[ (\vx-\vmu)(\vx-\vmu)^\top f(\vx)\right] \vSigma^{-1} = \myexpect \left[ f(\vx) \right] \vSigma^{-1}  + \myexpect \left[ \nabla_{\text{\vx}}^2 f(\vx)\right].
\]
Substitute this back into \eqref{eq:price-score}. The $\vSigma^{-1}\myexpect \left[ f(\vx) \right]$ terms cancel, and we obtain
\[
  \frac{\partial}{\partial \text{\vSigma}_{ij}} \myexpect \left[ f(\vx) \right] = \tfrac12 \myexpect \left[ \partial_{\text{\vx}_i}\partial_{\text{\vx}_j} f(\vx)\right].
\]
Equivalently, in matrix form,
\[
  \nabla_{\text{\vSigma}} \myexpect \left[ f(\vx) \right]  = \tfrac12 \myexpect \left[ \nabla_{\text{\vx}}^2 f(\vx)\right],
\]
which is Price's theorem.
\end{proof}

\begin{remark}[Symmetry in $\vSigma$]
In the proof of Price's theorem we differentiate with respect to the entries $\vSigma_{ij}$ treating $\vSigma$ as an unconstrained matrix. The right-hand side of the final expression for the gradient is symmetric because the Hessian $\nabla_{\text{\vx}}^2 f(\vx)$ is symmetric for $C^2$ functions $f$. Thus the gradient naturally lies in the
space of symmetric matrices, and the formula is consistent with the covariance constraint $\vSigma = \vSigma^\top$.
\end{remark}

\section{Elliptical Laws}\label{app:elliptical}


\paragraph{Generalized Pearson Type II. } 
An important subclass of elliptical distributions are the Pearson Type II distributions. 
A fundamental property characterizing  these distributions, is the stochastic representation of a $D$-dimensional spherical Pearson Type II random vector $\vz$ as $\vz \stackrel{d}{=} r \vu$, where  $\vu$ has uniform distribution on the unit sphere surface in $\real^D$, and the sign $d$ indicates the same distribution. The random variables $r \stackrel{d}{=} \| \vz \|$ and $\vu \stackrel{d}{=} \frac{\text{\vz}}{\| \text{\vz} \|}$ are independent and therefore
the distribution of $\vz$ is fully determined by that of its squared length $r^2 =  \vz^\top \vz$. The random variables $r$ and
\[
s(\vx) = r^2 = (\vx - \vmu)^\top \vSigma^{-1}(\vx-\vmu)
\]
are called the radial and squared radial parts of the random vector $\vx$. Their distributions are defined by the specific Pearson Type II density generator function in $\eqref{eq:genPearsonII}$, as
\begin{align*}
f_r(r) = \frac{2 \pi^{D/2}}{\Gamma(D/2)} r^{D-1} g(r^2),
\quad r > 0, \qquad
f_s(s)  = \frac{\pi^{D/2}}{\Gamma(D/2)} s^{D/2-1} g(s).
\end{align*}
For further details see \cite{johnson1987multivariate, fang1990symmetric,  rezaei2018inferences, arellano2001some}.


\begin{proof}[Proof of Lemma~\ref{lem:q-gaussian-pearsonII}]

Expressing the exponent $m$ of the Pearson type II density in terms of the entropic index $q < 1$ of the $q$-Gaussian density as $m=1/1-q$, and choosing the support radius $R=R(D, q)$ to be 
\begin{equation}\label{eq:R2q}
R^2 =\left[
  \frac{\Gamma \left(\frac{D}{2}+\frac{2-q}{1-q}\right)}
       {\pi^{D/2} \Gamma \left(\frac{2-q}{1-q}\right)}
  \left(\frac{2}{1-q}\right)^{\frac{1}{1-q}} 
  \right]^{\frac{2(1-q)}{2+D(1-q)}}  , 
\end{equation}
the density generator of the Pearson type II distribution, given in $\eqref{eq:genPearsonII}$, takes the form
\begin{equation}
g(s) = \left(\frac{1-q}{2} \left(R^2-s \right)\right)^{1/(1-q)}_+ .
\end{equation}
The corresponding Pearson type~II density is given by
\begin{align*}
p(\vx) &=  |\vSigma|^{-1/2} \left[ \frac{1-q}{2} 
\left(R^2-(\vx - \vmu)^\top \vSigma^{-1}(\vx-\vmu) \right)\right]^{1/(1-q)} q < 1, (\vx - \vmu)^\top \vSigma^{-1}(\vx-\vmu) < R^2.
\end{align*}
This can be equivalently expressed in the form of the $q$-Gaussian density as
\begin{align*}
p(\vx) &= \left[1+(1-q)\left( -\frac{1}{1-q} 
+\frac{|\vSigma|^{-(1-q)/2}}{2}\left( 
R^2-(\vx - \vmu)^\top \vSigma^{-1}(\vx-\vmu)\right)\right) \right]_+^{1/(1-q)} \\
&= \exp_{q} \left[-\frac{1}{1-q} +\frac{|\vSigma|^{-(1-q)/2}}{2}\left( 
R^2-(\vx - \vmu)^\top \vSigma^{-1}(\vx-\vmu)\right)
\right].
\end{align*}

Introducing for convenience a transformation of the scale parameter $\tSigma = |\vSigma|^{\frac{1-q}{2}} \vSigma$ and using the fact that $|\vSigma| = |\tSigma|^{\frac{2}{(1-q)D+2}}$, \citet[Def.~15]{martins2022sparse} define the density parametrized in terms of $\vmu$ and $\tSigma$ as
\[ \vx \sim N_{q} (\vmu , \tSigma) \Rightarrow  
p(\vx) = \exp_{q} \left[-\frac{1}{1-q} +\frac{R^2 |\tSigma|^{-\frac{1-q}{D(1-q)+2}}}{2} -\frac{(\vx - \vmu)^\top \tSigma^{-1}(\vx-\vmu)}{2} \right]. 
\]
\citet{johnson1987multivariate} provides a detailed discussion of the effect of the shape parameter $q$. For $q \to 1$, the radius $R \to \infty$, $\exp_{q}$ converges to the ordinary exponential, and we recover the multivariate normal $N (\vmu , \vSigma)$.
\end{proof}

\paragraph{Associated laws as escort distributions.} Using the density generator $g(\cdot)$ of the bounded-support $q$-Gaussian, given in \eqref{eq:gen-q}, the cumulative generator function is defined as the integral of the original generator:
\begin{align*}
G(s) &: =\int_{s}^{R^2} g(t) dt 
= \left(\tfrac{1-q}{2}\right)^{\frac{1}{1-q}} \frac{1-q}{2-q} (R^2-s)^{\frac{2-q}{1-q}}
\propto (R^2-s)^{m+1},\quad m:= \frac{1}{1-q}, 
\end{align*}
which induces the first associated law with density
\begin{align*}
p^\star(\vx)
&=\frac{|\vSigma|^{-1/2} G \left(s(\vx) \right)}
{\displaystyle \frac{\pi^{D/2}}{\Gamma(\frac D2)} \int_{0}^{R^2} s^{\frac D2-1} G(s) ds}
  \propto   |\vSigma|^{-1/2} \left(R^2-s(\vx) \right)^{m+1}_+, \\
\end{align*}
where $s(\vx):=(\vx-\vmu)^\top \vSigma^{-1}(\vx-\vmu)$. By construction, $p$ and $p^\star$ share the same location-scale parameters $(\vmu,\vSigma)$ and the same bounded support $\mathrm{supp} = \{\vx: s(\vx)<R^2\}$; the associated law is a new Pearson Type~II distribution on the same support set, with  exponent increased by $1$ as a consequence of the power-law Pearson II form.

We can equivalently express it as normalized reweighting of the base law $p(\vx)$:
\begin{equation}\label{eq:reweight}
p^\star(\vx)  =  \frac{\left(R^2-s(\vx) \right) p(\vx)}{\myexpect_{p} \left[ R^2-s(\vx) \right]},
\end{equation}

\paragraph{Escort interpretation of the associated laws.}
The first associated law coincides with the $(2-q)$-escort of the base density. Indeed,
\[
p^\star(\vx) \propto\ (R^2-s(\vx))^{m+1}_+ = \left((R^2-s(\vx))^{m}_+\right)^{2-q} \propto p(\vx)^{ 2-q},
\]
since $m+1 = \frac{2-q}{1-q} = (2-q)m$. More generally, the $k$-th associated law satisfies
\[
p^{\langle k\rangle}(\vx)\ \propto\ (R^2-s(\vx))^{m+k}_+ \propto  p(\vx)^{ 1+k(1-q)}, \qquad k=0,1,2,\dots,
\]
so the escort order increases by $(1-q)$ at each step. All these laws share the same bounded support and are progressively more centrally concentrated, with a sharper peak at $\vmu$ and vanishing faster near the boundary than $p(\vx)$.

\section{Proof of the bounded-support Stein identity}\label{app:stein-proof}

We prove the bounded-support Stein for a general test function $t(\vx)$ rather than the more specific $f(\vx)$ which we use for simplicity in the main statement of Theorem~\ref{thm:stein-bounded}. We assume $p$ is a bounded-support $q$-Gaussian density with parameters $(\vmu, \vSigma,q)$ as in Lemma~\ref{lem:q-gaussian-pearsonII}, and $p^\star$ is its first associated (cumulative-generator) law. For any almost everywhere differentiable $t:\real^D \to \real$ with $\myexpect_{p^\star} \left\| \nabla t(\vx) \right\| < \infty$,
\begin{equation}\label{eq:stein-pearsonII-expect-2}
\myexpect_{p} \left[ (\vx-\vmu) t(\vx) \right] 
= \Cov_{p}(\vx) \myexpect_{p^\star} \left[ \nabla t(\vx) \right]
= \frac{\myexpect_{p} \left[ r^2 \right]}{D} \vSigma \myexpect_{p^\star} \left[ \nabla t(\vx) \right],
\end{equation}
where $r^2 = s(\vx) = (\vx-\vmu)^\top \vSigma^{-1}(\vx-\vmu)$ and $\myexpect_{p} \left[ r^2 \right]$ are given in Lemma~\ref{lem:radial-moments-associated}. 

\begin{proof}[Proof of Theorem~\ref{thm:stein-bounded}]
Let $\vz = \vSigma^{-1/2}(\vx-\vmu)$, so that
\begin{align}\label{eq:escort-in-z}
p(\vz) \propto\ (R^2-\| \vz \|^2)^m \mathbbm{1}_{\| \vz \| < R}, \qquad
p^\star(\vz)  = \frac{(R^2-\| \vz \|^2) p(\vz)}{\myexpect_{p} \left[ R^2-\| \vz \|^2 \right]} \propto\ (R^2-\| \vz \|^2)^{m+1} \mathbbm{1}_{\| \vz \| < R},
\end{align}
and define $T(\vz):= t(\vmu + \vSigma^{1/2} \vz)$. To compute $\myexpect_{p^\star} \left[ \partial_{\text{\vz}_i} T(\vz) \right]$ first we integrate out $\vz_{i}$. Fix an index $i$. For each $\vz_{-i} \in \real^{D-1}$ with $\| \vz_{-i} \|<R$, set
\[
\rho(\vz_{-i}) := \sqrt{R^2-\| \vz_{-i} \|^2} \in (0,R],
\]
so that the admissible range for $\vz_i$ is exactly
\[
\vz_i\in I( \vz_{-i}) := [-\rho(\vz_{-i}), \rho(\vz_{-i})],
\]
since $\| \vz \|^2  < R^2 \implies \| \vz_{-i} \|^2 + \vz_i^2  < R^2 \implies |\vz_i | < \rho(\vz_{-i}) $. For each fixed $\vz_{-i}$, we apply one-dimensional integration by parts in the variable $\vz_i$ on the interval $I(\vz_{-i})$:
\begin{align*}
& \int_{-\rho(\text{\vz}_{-i})}^{\rho(\text{\vz}_{-i})} \partial_{\text{\vz}_i}  T(\vz_{-i}, \vz_i) (R^2 - \| \vz_{-i} \|^2 - \vz_i^2)^{m+1} d \vz_i \\
& = \left[T(\vz_{-i}, \vz_i) (R^2-\| \vz_{-i} \|^2 - \vz_i^2)^{m+1} \right]_{\text{\vz}_i = -\rho(\text{\vz}_{-i})}^{\rho( \text{\vz}_{-i})} 
+ \int_{-\rho(\text{\vz}_{-i})}^{\rho(\text{\vz}_{-i})} T(\vz_{-i}, \vz_i) \partial_{\text{\vz}_i} (R^2 - \| \vz_{-i} \|^2 - \vz_i^2)^{m+1} d \vz_i.  
\end{align*}
Because $m+1>0$, the polynomial $(R^2 - \| \vz_{-i} \|^2 - \vz_i^2)^{m+1}$ vanishes at $\vz_i = \pm \rho(\vz_{-i})$, and the boundary term is exactly zero for each slice. Differentiating the polynomial term gives
\[
\partial_{\text{\vz}_i} (R^2-\| \vz_{-i} \|^2 - \vz_i^2)^{m+1} = -2 (m+1) \vz_i (R^2-\| \vz_{-i} \|^2 - \vz_i^2)^{m}.
\]
Hence,
\begin{align}\label{eq:inner}
\int_{-\rho(\text{\vz}_{-i})}^{\rho(\text{\vz}_{-i})} \partial_{\text{\vz}_i}  T(\vz_{-i}, \vz_i) (R^2-\| \vz_{-i} \|^2 - \vz_i^2)^{m+1} d \vz_i
= \int_{-\rho(\text{\vz}_{-i})}^{\rho(\text{\vz}_{-i})} T(\vz_{-i}, \vz_i) 2(m+1) \vz_i (R^2-\| \vz_{-i} \|^2 - \vz_i^2)^{m} d \vz_i.
\end{align}

Next, we integrate over all $\vz_{-i}$ with $\| \vz_{-i} \| < R$. The left-hand side of \eqref{eq:inner} becomes
\begin{align*}
\int_{\text{\real}^{D-1}} \int_{\text{\real}} \left(\partial_{\text{\vz}_i} T(\vz_{-i}, \vz_i) \right) (R^2 - \| \vz_{-i} \|^2 - \vz_i^2)^{m+1}_+ d \vz_i d\vz_{-i} &= \int_{\real^{D}} (\partial_{\text{\vz}_i} T(\vz)) (R^2 - \| \vz \|^2 )^{m+1}_+ d\vz \\
&= C^\star \myexpect_{p^\star} \left[ \partial_{\text{\vz}_i} T(\vz) \right], 
\end{align*}
where $C^\star$ is the normalizing constant of $p^\star$. Finally, the right-hand side of \eqref{eq:inner} becomes
\[
2(m+1) \int_{\| \text{\vz} \| < R} T(\vz) \vz_i (R^2-\| \vz \|^2)^m d\vz = 2(m+1) C \myexpect_{p} \left[ \vz_i T(\vz) \right],
\]
where $C$ is the normalizing constant of $p$. Hence
\begin{equation}\label{eq:ibp-core-explicit}
\myexpect_{p^\star} \left[ \partial_{\text{\vz}_i} T(\vz) \right] = \frac{2(m+1) C}{C^\star} \myexpect_{p} \left[ \vz_i T(\vz) \right].
\end{equation}
Using the relation
\[
\myexpect_{p} \left[ R^2-\|\vz\|^2 \right] = \frac{1}{C} \int (R^2-\|\vz\|^2) (R^2-\|\vz\|^2)^m d\vz = \frac{C^\star}{C},
\]
we get $\dfrac{C}{C^\star} = \dfrac{1}{\myexpect_{p} \left[ R^2-\|\vz\|^2 \right]}$, so \eqref{eq:ibp-core-explicit} simplifies to
\begin{equation}\label{eq:ibp-core}
\myexpect_{p^\star} \left[ \partial_{\text{\vz}_i} T(\vz) \right] = \frac{2(m+1)}{\myexpect_{p} \left[ R^2-\|\vz\|^2 \right]} \myexpect_{p} \left[ \vz_i T(\vz) \right].
\end{equation}

Vectorizing over $i$, we obtain
\[
\myexpect_{p^\star} \left[ \nabla_{\text{\vz}} T(\vz) \right] = \frac{2(m+1)}{\myexpect_{p} \left[ R^2-\|\vz\|^2 \right]} \myexpect_{p} \left[ \vz T(\vz) \right].
\]

\paragraph{Return to $\vx$-space.} Recall $\vz= \vSigma^{-1/2}(\vx-\vmu)$ and $\nabla_{\text{\vz}} T(\vz) = \vSigma^{1/2} \nabla_{\text{\vx}} t(\vmu + \vSigma^{1/2} \vz)$. Taking expectations and multiplying by $\vSigma^{1/2}$,
\[
\myexpect_{p} \left[ (\vx-\vmu) t(\vx) \right] = \vSigma^{1/2} \myexpect_{p} \left[ \vz T(\vz) \right]
= \frac{\myexpect_{p} \left[ R^2-\|\vz\|^2 \right]}{2(m+1)} \vSigma^{1/2} \myexpect_{p^\star} \left[ \nabla_{\text{\vz}} T(\vz) \right]
= \frac{\myexpect_{p} \left[ R^2-\|\vz\|^2 \right]}{2(m+1)}  \vSigma \myexpect_{p^\star} \left[ \nabla_{\text{\vx}} t(\vx) \right].
\]
From Lemma~\ref{lem:radial-moments-associated},
\[
\frac{\myexpect_{p} \left[ r^2 \right]}{D} = \frac{\myexpect_{p} \left[ R^2-r^2 \right]}{2(m+1)},
\]
so the coefficient equals $\frac{\myexpect_{p} \left[ r^2 \right]}{D}$ and we obtain \eqref{eq:stein-pearsonII-expect}. Using \eqref{eq:reweight} to express $\myexpect_{p^\star}$ in terms of $\myexpect_{p}$ yields \eqref{eq:final-p-only}.
\end{proof}

\section{Proofs of $q$-Bonnet, $q$-Price, and variance bounds}\label{app:q-bonnet-proof}

Similarly to Appendix~\ref{app:stein-proof} we specify in the following proofs the exact form of the test function $t(\vx)$ that we use, and note that in the main statements of the Theorems we use $f(\vx)$ instead for simplicity. 

\subsection{$q$-Bonnet}
\label{app:q-bonnet-proof-1}

Let $p$ be a bounded-support $q$-Gaussian density $N_q(\vmu,\vSigma)$ from Lemma~\ref{lem:q-gaussian-pearsonII}, with $q<1$ and $m := 1/(1-q) > 0$. Let $f:\real^D \to \real$ be $C^1$ on an open set containing $\{s(\vx) \le R^2\}$, where $s(\vx) := (\vx-\vmu)^\top \vSigma^{-1}(\vx-\vmu)$, and assume $\myexpect_{p} \left\| \nabla f(\vx)\right\| < \infty$. We show that

\[
\nabla_{\text{\vmu}} \myexpect_{p} \left[ f(\vx) \right] = \myexpect_{p} \left[ \nabla f(\vx) \right].
\]

\begin{proof}[Proof of Theorem~\ref{thm:q-bonnet}]
Since $p(\vx)\propto |\vSigma|^{-1/2}\left( R^2-s(\vx) \right)^m_+$, we have
\[
\log p(\vx) = -\tfrac12\log|\vSigma| + m \log \left( R^2-s(\vx) \right)+\text{const},\quad
\nabla_{\text{\vmu}} s(\vx) = -2 \vSigma^{-1}(\vx-\vmu).
\]
Hence
\[
\nabla_{\text{\vmu}}\log p(\vx)
= m \frac{\nabla_{\vmu}\left( R^2-s(\vx) \right)}{R^2-s(\vx)}
= \frac{2m}{R^2-s(\vx)} \vSigma^{-1}(\vx-\vmu).
\]
Differentiating under the integral is justified by dominated convergence on the bounded support together with $m>0$, so we obtain the score identity 
\begin{equation}\label{eq:score-id}
\nabla_{\text{\vmu}} \myexpect_{p} \left[ f(\vx)\right]
= \myexpect_{p} \left[ f(\vx) \nabla_{\text{\vmu}} \log p(\vx)\right]
= 2m \myexpect_{p} \left[ \frac{f(\vx)}{R^2 - s(\vx)} \vSigma^{-1} (\vx - \vmu) \right].
\end{equation}

Apply the $p$-only Stein identity \eqref{eq:final-p-only} with $t(\vx) = \dfrac{ f(\vx)}{ R^2-s(\vx)}$. Substituting into the right-hand side of \eqref{eq:final-p-only} by the quotient rule and using $\nabla s(\vx) = 2 \vSigma^{-1}(\vx - \vmu)$ gives
\[
(R^2-s(\vx)) \nabla \left(\frac{f(\vx)}{R^2-s(\vx)}\right) 
= \nabla f(\vx) + 2 \frac{f(\vx)}{R^2-s(\vx)} \vSigma^{-1}(\vx-\vmu).
\]
Hence \eqref{eq:final-p-only} yields
\begin{equation}\label{eq:steineq-fixpoint}
\vSigma \myexpect_{p} \left[ \frac{f(\vx)}{R^2-s(\vx)} \vSigma^{-1}(\vx - \vmu) \right]
= \frac{\myexpect_{p} \left[ r^2 \right]}{D} \vSigma
\frac{\myexpect_{p} \left[ \nabla f(\vx) \right] + 2 \myexpect_{p} \left[ \frac{f(\vx)}{R^2-s(\vx)} \vSigma^{-1}(\vx-\vmu) \right]}{\myexpect_{p} \left[ R^2-s(\vx) \right]}.
\end{equation}
From the moment formulas in Lemma~\ref{lem:radial-moments-associated}, we have
\[
\frac{\myexpect_{p} \left[ r^2 \right]}{D \myexpect_{p} \left[ R^2-s(\vx) \right]} = \frac{1}{2(m+1)}.
\]
Thus \eqref{eq:steineq-fixpoint} yields
\begin{align*}
 \myexpect_{p} \left[ \frac{f(\vx)}{R^2-s(\vx)} \vSigma^{-1}(\vx - \vmu) \right] 
 &= \frac{1}{2(m+1)} \left(\myexpect_{p} \left[ \nabla f(\vx) \right]
    + 2 \myexpect_{p} \left[ \frac{f(\vx)}{R^2-s(\vx)} \vSigma^{-1}(\vx - \vmu) \right] \right) \\
 &= \frac{1}{2m} \myexpect_{p} \left[ \nabla f(\vx) \right].
\end{align*}
Substitute this into \eqref{eq:score-id}:
\[
\nabla_{\text{\vmu}} \myexpect_{p} \left[ f(\vx) \right] = 2m  \myexpect_{p} \left[ \frac{f(\vx)}{R^2-s(\vx)} \vSigma^{-1}(\vx - \vmu) \right]
= \myexpect_{p} \left[ \nabla f(\vx) \right].
\]

\end{proof}

\subsection{$q$-Price}
\label{app:q-bonnet-proof-2}

\begin{lemma}[Matrix calculus]\label{lem:matrix-calculus}
Let $\vSigma \succ 0$. We use the Frobenius inner product $A:B := \mathrm{tr}(A^\top B)$. For $s(\vx)=(\vx-\vmu)^\top \vSigma^{-1}(\vx-\vmu)$ and the elementary matrix $\vE_{ij}$  that has 1 in the $(i,j)$ position and 0 elsewhere: 
\begin{align}
&\frac{\partial}{\partial \vSigma_{ij}}\log| \vSigma| = (\vSigma^{-1})_{ji}, \label{eq:dlogdet}\\
&\frac{\partial s(\vx)}{\partial \vSigma_{ij}}
= -\left(\vSigma^{-1}\vE_{ji} \vSigma^{-1}\right):\left( (\vx-\vmu) (\vx-\vmu)^\top \right)
= -\left(\vSigma^{-1}(\vx-\vmu)(\vx-\vmu)^\top \vSigma^{-1}\right)_{ji}.
\label{eq:dQ}
\end{align}
Moreover, for any matrix $\vH$,
\begin{align}
&\left(\vSigma^{-1} \vE_{ji}\vSigma^{-1}\right):\vSigma = (\vSigma^{-1})_{ji},\label{eq:contr1}\\
&\left(\vSigma^{-1} \vE_{ji}\vSigma^{-1}\right):(\vSigma \vH \vSigma) = \vH_{ji}. \label{eq:contr2}
\end{align}
\end{lemma}

As before, let $p$ be the bounded-support $q$-Gaussian density $N_q(\vmu,\vSigma)$ with $q<1$ and shape $m=\tfrac{1}{1-q}>0$, supported on $\{ s(\vx) < R^2 \}$. For the proof of Theorem~\ref{thm:q-price} we assume $f:\real^D \to \real$ is $C^2$ on an open set containing $\{ s(\vx) \le R^2\}$ and that
$\myexpect_p \| \nabla f(\vx) \| + \myexpect_p\| \nabla^2 f(\vx) \|_{F}~<~\infty$. We define $M$ and recall the form of the covariance
\[
M := \myexpect_p \left[R^2-s(\vx) \right],\qquad 
\Cov_p(\vx) = \frac{\myexpect_p[r^2]}{D} \vSigma.
\]
As usual, $p^\star$ denotes the first associated law of $p$, and so $\myexpect_{p^\star}[h] = \myexpect_p[(R^2-s)h]/M$ for any integrable $h$. Treating the entries $\vSigma_{ij}$ as independent parameters, the following identities hold for every $i,j \in \{1, \dots, D\}$:
\begin{align}\label{eq:q-price-associated}
\frac{\partial}{\partial \text{\vSigma}_{ij}} \myexpect_p \left[f(\vx)\right]
&= \frac{\myexpect_p[r^2]}{2D}
\mathbb E_{p^\star} \left[\frac{\partial^2 f(\vx)}{\partial x_i \partial x_j}\right],
\end{align}
Equivalently, in matrix form,
\begin{equation}\label{eq:q-price-matrix-1}
\nabla_{ \text{\vSigma} } \myexpect[f(\vx)]
= \frac{\myexpect_p[r^2]}{2D} \myexpect_{p^\star}[\nabla_{\text{\vx}}^2 f(\vx)]
\end{equation}
In the Gaussian limit $q \uparrow 1$ (so $m \to \infty$, $R \to \infty$), $p^\star = p$ and $\myexpect_p[r^2]/D\to 1$, so \eqref{eq:q-price-associated} reduces to the classical Price theorem
$\frac{\partial}{\partial \text{\vSigma}_{ij}}\myexpect[f(\vx)]=\tfrac12 \myexpect[\partial_{x_i}\partial_{x_j} f(\vx)]$.

\begin{proof}[Proof of Theorem~\ref{thm:q-price}]

\noindent \textbf{Score trick w.r.t. $\vSigma$. } By differentiation under the integral (justified by the assumptions and bounded support),
\[
\frac{\partial}{\partial \text{\vSigma}_{ij}} \myexpect_p[f(\vx)]
= \myexpect_p \left[f(\vx) \frac{\partial}{\partial \text{\vSigma}_{ij}} \log p(\vx)\right].
\]
For $p(\vx) \propto |\vSigma|^{-1/2} (R^2-s(\vx))^m_+$,
\[
\log p(\vx) = -\tfrac12\log |\vSigma| + m \log(R^2-s(\vx)) + \text{const.}
\]
Using Lemma~\ref{lem:matrix-calculus},
\[
\frac{\partial}{\partial \text{\vSigma}_{ij}} \log p(\vx) 
= -\tfrac12(\vSigma^{-1})_{ji} + \frac{m}{R^2 - s(\vx)}\left(\vSigma^{-1}(\vx-\vmu)(\vx-\vmu)^\top \vSigma^{-1}\right)_{ji}.
\]
Define the matrix
\[
B := \myexpect_p \left[\frac{(\vx-\vmu)(\vx-\vmu)^\top}{R^2-s(\vx)} f(\vx) \right].
\]
Then, in matrix form,
\begin{equation}\label{eq:score-form}
\nabla_{\text{\vSigma}} \myexpect_p[f(\vx)]
= -\tfrac12 \vSigma^{-1} \myexpect_p[f(\vx)] + m \vSigma^{-1} B \vSigma^{-1}.
\end{equation}

\noindent \textbf{Identify $B$ via the Stein identity. } Invoke the bounded-support Stein identity in its $p$-only form:
\[
\myexpect_p \left[ (\vx-\vmu) t(\vx)^\top \right]
=\frac{\myexpect_p [r^2]}{D} \vSigma
\frac{\myexpect_p \left[(R^2 - s(\vx)) \nabla t(\vx)^\top \right]}{M},
\qquad M := \myexpect_p[R^2 - s(\vx)].
\]
Choose the vector test function
\[
t(\vx) := \frac{f(\vx)}{R^2 - s(\vx)} \vSigma^{-1}(\vx - \vmu).
\]
\emph{Left-hand side.} Using the definition of $B$,
\[
\myexpect_p \left[(\vx - \vmu) t(\vx)^\top \right]
= \myexpect_p \left[\frac{(\vx-\vmu) (\vx-\vmu)^\top}{R^2-s(\vx)} f(\vx) \right] \vSigma^{-1}
= B \vSigma^{-1}.
\]
\emph{Right-hand side.} Compute 
\[
(R^2-s) \nabla t = \vSigma^{-1}(\vx-\vmu) \nabla f(\vx)^\top
+ 2 \vSigma^{-1}(\vx-\vmu)(\vx-\vmu)^\top \vSigma^{-1} \frac{f(\vx)}{R^2-s(\vx)}
+ f(\vx) \vSigma^{-1},
\]
using $\nabla s(\vx) = 2\vSigma^{-1}(\vx-\vmu)$ and the product rule. Taking $\myexpect_p$ and substituting into Stein’s identity yields
\[
B \vSigma^{-1} = c \left( \vSigma \myexpect_p[\nabla f(\vx) (\vx-\vmu)^\top] \vSigma^{-1}
+ 2 B \vSigma^{-1} + \myexpect_p[f(\vx)] \vI \right),
\qquad c:=\frac{\myexpect_p[r^2]}{D M}.
\]
Right-multiplying by $\vSigma$ and rearranging gives
\[
(1-2c) B = c \left(\vSigma \myexpect_p[\nabla f(\vx) (\vx-\vmu)^\top]
+ \myexpect_p[f(\vx)] \vSigma \right).
\]
For the bounded-support $q$-Gaussian, Lemma~\ref{lem:radial-moments-associated} gives $c=\myexpect_p[r^2]/(D M)=1/(2(m+1))$, hence
\begin{equation}\label{eq:B-solved}
B = \frac{1}{2m} \left(\vSigma \myexpect_p[\nabla f(\vx) (\vx-\vmu)^\top]
+ \myexpect_p[f(\vx)] \vSigma \right).
\end{equation}

\noindent \textbf{Insert $B$ back into the score form. } Substitute \eqref{eq:B-solved} into \eqref{eq:score-form}:
\[
\nabla_{ \text{\vSigma}} \myexpect_p[f(\vx)]
= -\tfrac12 \vSigma^{-1} \myexpect_p[f(\vx)]
+\frac{m}{2m} \left( \vSigma^{-1} \vSigma \myexpect_p[\nabla f(\vx) (\vx-\vmu)^\top] \vSigma^{-1}
+\vSigma^{-1} \myexpect_p[f(\vx)] \vSigma \vSigma^{-1} \right).
\]
The $\pm \tfrac12 \vSigma^{-1} \myexpect_p[f(\vx)]$ terms cancel, leaving
\begin{equation}\label{eq:first-order-form}
\nabla_{ \text{\vSigma}} \myexpect_p[f(\vx)]
= \frac12 \myexpect_p \left[(\vx-\vmu) \nabla f(\vx)^\top \right] \vSigma^{-1} .
\end{equation}

\noindent \textbf{Second Stein step to reach the Hessian. } Apply the (vector) Stein identity with the scalar tests $t_j(\vx) := \partial_{x_j} f(\vx)$, $j=1, \dots, D$, and stack:
\[
\myexpect_p \left[(\vx-\vmu) \nabla f(\vx)^\top\right]
= \frac{\myexpect_p[r^2]}{D} \vSigma 
\myexpect_{p^\star} \left[\nabla_{\text{\vx}}^2 f(\vx) \right].
\]
Substitute this into \eqref{eq:first-order-form} to obtain
\[
\nabla_{ \text{\vSigma}} \myexpect_p[f(\vx)]
= \frac{\myexpect_p[r^2]}{2D} \myexpect_{p^\star} \left[\nabla_{\text{\vx}}^2 f(\vx)\right],
\]
which is \eqref{eq:q-price-matrix}. 

\end{proof}

\begin{remark}[Symmetry in $\vSigma$]
We differentiated with respect to the entries $\vSigma_{ij}$ treating $\vSigma$ as an unconstrained matrix. The right-hand side of \eqref{eq:q-price-matrix} is symmetric because
$\nabla_{\text{\vx}}^2 f(\vx)$ is symmetric for $C^2$ functions $f$. Hence the gradient
$\nabla_{\text{\vSigma}} \myexpect_p[f(\vx)]$ naturally lies in the space of symmetric matrices and is consistent with the constraint $\vSigma = \vSigma^\top$.
\end{remark}

\subsection{Variance bounds}
\label{app:q-bonnet-proof-3}

\begin{proof}[Proof of Proposition~\ref{prop:bv}]
It is straightforward to verify unbiasedness: $\myexpect \left[ \widehat \vg \right] = \myexpect_{p^\star} \left[ \nabla t(\vx) \right]$ and $\myexpect \left[ \widehat \vH \right] = \myexpect_{p^{\star}} \left[ \nabla^2 f(\vx) \right]$. 

Note that $\widehat g_j$ is scalar, and let
\[
Y_k := \frac{(R^2-s(\vx_k)) \partial_j t(\vx_k)}{\myexpect_{p} \left[ R^2-s(\vx) \right]}.
\]
Since $0 < R^2-s(\vx) \le R^2$ and $|\partial_j t(\vx)| \le \left\| \nabla t(\vx)\right\| \le C_1$, hence
\[
|Y_k| \le \frac{R^2 C_1}{\myexpect_{p} \left[ R^2 - s(\vx) \right]} =: B_g.
\]
Therefore $Y_k \in [-B_g,B_g]$ almost surely, and by Popoviciu's inequality on bounded-range variances,
\[
\Var(Y_k) \le B_g^2.
\]
Since $\widehat g_j = \frac1S \sum_{k=1}^n Y_k$ with iid $Y_k$,
\[
\Var(\widehat g_j) = \frac{\Var(Y_1)}{S} \le \frac{B_g^2}{S}
= \frac{1}{S} \left( \frac{R^2 C_1}{M} \right)^2.
\]

Next, fix $(i,j)$ and define
\[
Z_k := \frac{(R^2-s(\vx_k)) (\nabla^2 f(\vx_k))_{ij}}{M}.
\]
Since $0 < (R^2-s(\vx)) \le R^2$ and $|(\nabla^2 f(\vx))_{ij}| \le \|\nabla^2 f(\vx)\|_{\mathrm{op}} \le C_2$, hence
\[
|Z_k| \le \frac{R^2 C_2}{M} =: B_H,
\]
so $Z_k \in [-B_H,B_H]$ almost surely, and again by Popoviciu, $\Var(Z_k) \le B_H^2$, and
\[
\Var \left((\widehat H)_{ij} \right)
=\Var \left(\frac1S\sum_{k=1}^n Z_k \right)
=\frac{\Var(Z_1)}{S} \le \frac{B_H^2}{S}
= \frac{(R^2 C_2)^2}{S M^2}.
\]

Summing entrywise variances gives
\[
\myexpect \left[ \left\| \widehat \vH - \myexpect \left[ \widehat \vH \right] \right\|_F^2 \right]
= \sum_{i,j} \Var\left((\widehat \vH)_{ij}\right) 
\le \frac{D^2 B_H^2}{S}.
\]
Since $\|\cdot\|_{\mathrm{op}} \le \|\cdot\|_F$, the same bound implies
\[
\myexpect \left[ \left\| \widehat \vH - \myexpect \left[ \widehat \vH \right] \right\|_{\mathrm{op}} \right] 
\le \frac{D B_H}{\sqrt{S}}.
\]

A sharper dimension dependence follows from matrix Hoeffding: defining centered
\[
\vA_k := \frac{(R^2-s(\vx_k)) \nabla^2 f(\vx_k)}{M}-\myexpect \left[ \widehat \vH \right]
\]
implies $\|\vA_k\|_{\mathrm{op}}\le 2 B_H$. Hence, since
$\frac{1}{S} \sum_{k=1}^S \vA_k = \widehat \vH-\myexpect \left[ \widehat \vH \right]$, we have 
\[
\Pr \left(\left\|\widehat \vH - \myexpect \left[ \widehat \vH \right] \right\|_{\mathrm{op}}\ge t \right)
 \le 2D \exp \left( -\frac{S t^2}{8 (2B_H)^2}\right),
\]
which integrates to 
$\myexpect \left[ \left\| \widehat \vH - \myexpect \left[ \widehat \vH \right] \right\|_{\mathrm{op}} \right] 
\le C_3 B_H \sqrt{\frac{\log(2D)}{S}}$
for a universal $C_3$: by the same tail-integration argument as above, and splitting at $t_0>0$,
\begin{align*}
\myexpect \left[ \left\|\widehat \vH - \myexpect \left[ \widehat \vH \right] \right\|_{\mathrm{op}} \right]
= \int_{0}^{\infty}\Pr( \left\|\widehat \vH - \myexpect \left[ \widehat \vH \right] \right\|_{\mathrm{op}} \ge t) dt
\le   t_0 +  \int_{t_0}^{\infty} 2D \exp  \left(-\frac{S t^2}{8(2B_H)^2}\right) dt.
\end{align*}
With $a := \dfrac{S}{8(2B_H)^2}$ and the Gaussian tail bound 
$\int_{x}^{\infty}e^{-a t^2}dt \le \dfrac{1}{2a x}e^{-a x^2}$,
\[
\myexpect \left[ \left\|\widehat \vH - \myexpect \left[ \widehat \vH \right] \right\|_{\mathrm{op}} \right] \le t_0 + 2D \cdot \frac{1}{2a t_0} e^{-a t_0^2}.
\]
Choose $t_0:=\sqrt{\dfrac{\log(2D)}{a}}$ so that $e^{-a t_0^2}=1/(2D)$. Then
\[
\myexpect \left[ \left\|\widehat \vH - \myexpect \left[ \widehat \vH \right] \right\|_{\mathrm{op}} \right] \le \sqrt{\frac{\log(2D)}{a}} + \frac{1}{2\sqrt{a}}\frac{1}{\sqrt{\log(2D)}} 
\le C_4 \frac{1}{\sqrt{a}}\sqrt{\log(2D)}
= C_3 B_H \sqrt{\frac{\log(2D)}{S}},
\]
for universal constants $C_4,C_3>0$, as claimed.
\end{proof}

\end{document}